\documentclass{article}

\PassOptionsToPackage{numbers, compress}{natbib}

    \usepackage[preprint]{neurips_2020}
\usepackage{tablefootnote}

\usepackage{algorithm}
\usepackage[noend]{algorithmic}
\usepackage{enumitem}

\usepackage[utf8]{inputenc} 
\usepackage[T1]{fontenc}    
\usepackage{hyperref}       
\usepackage{url}            
\usepackage{booktabs}       
\usepackage{amsfonts}       
\usepackage{nicefrac}       
\usepackage{microtype}      
\usepackage{multicol}
\usepackage{multirow}
\usepackage{graphicx}
\usepackage{tablefootnote}
\usepackage{amsmath}
\usepackage{amsthm}
\DeclareMathOperator*{\argmax}{arg\,max}

\newtheorem{theorem}{Theorem}

\newtheorem{lemma}[theorem]{Lemma}

\setlist[itemize]{leftmargin=*}
\title{Provably More Efficient Q-Learning in the One-Sided-Feedback/Full-Feedback Settings }

\author{
    Xiao-Yue Gong\\
  Operations Research Center\\
  Massachusetts Institute of Technology\\
  Cambridge, MA 02142 \\
  \texttt{xygong@mit.edu}\\
  \And
  David Simchi-Levi\\
  Institute for Data, Systems, and Society, Department of Civil and\\
  Environmental Engineering, and Operations Research Center\\
  Massachusetts Institute of Technology\\
  Cambridge, MA 02139 \\
  \texttt{dslevi@mit.edu}\\
}

\begin{document}

\maketitle

\begin{abstract}
Motivated by the episodic version of the classical inventory control problem, we propose a new Q-learning-based algorithm, \emph{Elimination-Based Half-Q-Learning (HQL)}, that enjoys improved efficiency over existing algorithms for a wide variety of problems in the one-sided-feedback setting. We also provide a simpler variant of the algorithm, \emph{Full-Q-Learning (FQL)}, for the full-feedback setting. We establish that \emph{HQL} incurs $ \tilde{\mathcal{O}}(H^3\sqrt{ T})$ regret and \emph{FQL} incurs $\tilde{\mathcal{O}}(H^2\sqrt{ T})$ regret, where $H$ is the length of each episode and $T$ is the total length of the horizon. The regret bounds are not affected by the possibly huge state and action space. Our numerical experiments demonstrate the superior efficiency of \emph{HQL} and \emph{FQL}, and the potential to combine reinforcement learning with richer feedback models.
\end{abstract}

\section{Introduction}
Motivated by the classical operations research (OR) problem--inventory control, we customize Q-learning to more efficiently solve a wide range of problems with richer feedback than the usual bandit feedback. Q-learning is a popular reinforcement learning (RL) method that estimates the state-action value functions without estimating the huge transition matrix in a large MDP (\cite{qlearn}, \cite{jordan}). This paper is concerned with devising Q-learning algorithms that leverage the natural one-sided-feedback/full-feedback structures in many OR and finance problems.

{\bf Motivation} The topic of developing efficient RL algorithms catering to special structures is fundamental and important, especially for the purpose of adopting RL more widely in real applications. By contrast, most RL literature considers settings with little feedback, while the study of single-stage online learning for bandits has a history of considering a plethora of graph-based feedback models. We are particularly interested in the one-sided-feedback/full-feedback models because of their prevalence in many famous problems, such as inventory control, online auctions, portfolio management, etc. In these real applications, RL has typically been outperformed by domain-specific algorithms or heuristics. We propose algorithms aimed at bridging this divide by incorporating problem-specific structures into classical reinforcement earning algorithms.




\subsection{Prior Work}
 The most relevant literature to this paper is \cite{jin2018q}, who prove the optimality of Q-learning with Upper-Confidence-Bound bonus and Bernstein-style bonus in tabular MDPs. The recent work of \cite{1912.06366} improves upon \cite{jin2018q} when an aggregation of the state-action pairs with known error is given beforehand. Our algorithms substantially improve the regret bounds (see Table \ref{QLtable}) by catering to the full-feedback/one-sided-feedback structures of many problems. Because our regret bounds are unaffected by the cardinality of the state and action space, our Q-learning algorithms are able to deal with huge state-action space, and even continuous state space in some cases (Section \ref{conclusion}). Note that both our work and \cite{1912.06366} are designed for a subset of the general episodic MDP problems. We focus on problems with richer feedback; \cite{1912.06366} focus on problems with a nice aggregate structure known to the decision-maker.

The one-sided-feedback setting, or some similar notions, have attracted lots of research interests in many different learning problems outside the scope of episodic MDP settings, for example learning in auctions with binary feedback, dynamic pricing and binary search (\cite{weed}, (\cite{without}, \cite{dynamicpricing}, \cite{multidimbinary}). In particular, \cite{bandit} study the one-sided-feedback setting in the learning problem for bandits, using a similar idea of elimination. However, the episodic MDP setting for RL presents new challenges. Our results can be applied to their setting and solve the bandit problem as a special case.

The idea of optimization by elimination has a long history (\cite{elimination}). A recent example of the idea being used in RL is \cite{corruption} which solve a very different problem of robustness to adversarial corruptions. Q-learning has also been studied in settings with continuous states with adaptive discretization (\cite{metricadaptive}). In many situations this is more efficient than the uniform discretization scheme we use, however our algorithms' regret bounds are unaffected by the action-state space cardinality so the difference is immaterial.

Our special case, the full-feedback setting, shares similarities with \emph{the generative model} setting in that both settings allow access to the feedback for any state-action transitions (\cite{10.5555/3327345.3327425}). However, the generative model is a strong oracle that can query any state-action transitions, while the full-feedback model can only query for that time step after having chosen an action from the feasible set based on the current state, while accumulating regret.

\begin{table}[ht]
\caption{Regret comparisons for Q-learning algorithms on episodic MDP}
\label{QLtable}
\begin{center}
\begin{tabular}{llll}
\bf{Algorithm} & \bf{Regret} & \bf{Time} & \bf{Space}\\
\hline  Q-learning+Bernstein bonus \cite{jin2018q} & $\tilde{\mathcal{O}}(\sqrt{H^{3} S A T})$ & $\mathcal{O}(T)$ & $\mathcal{O}(SAH)$ \\
\hline Aggregated Q-learning \cite{1912.06366} & $\tilde{\mathcal{O}}(\sqrt{H^{4}MT}+\epsilon T)$ \tablefootnote{Here $M$ is the number of aggregate state-action pairs; $\epsilon$ is the largest difference between any pair of optimal state-action values associated with a common aggregate state-action pair.} & $\mathcal{O}(MAT)$ & $\mathcal{O}(MT)$\\
\hline Full-Q-learning (FQL) & $\tilde{\mathcal{O}}(\sqrt{H^{4} T})$ & $\mathcal{O}(SAT)$ & $\mathcal{O}(SAH)$\\
\hline Elimination-Based Half-Q-learning (HQL) & $\tilde{\mathcal{O}}(\sqrt{H^{6}T})$ & $\mathcal{O}(SAT)$ & $\mathcal{O}(SAH)$\\
\end{tabular}
\end{center}
\end{table}

\section{Preliminaries}
We consider an episodic Markov decision process, MDP($\mathcal{S}, \mathcal{A}, H, \mathbb{P}, r$), where $\mathcal{S}$ is the set of states with $|\mathcal{S}|=S$, $\mathcal{A}$ is the set of actions with $|\mathcal{A}|=A$, $H$ is the constant length of each episode, $\mathbb{P}$ is the unknown transition matrix of distribution over states if some action $y$ is taken at some state $x$ at step $h\in[H]$, and $r_h:\mathcal{S}\times \mathcal{A} \rightarrow [0,1]$ is the reward function at stage $h$ that depends on the environment randomness $D_h$. In each episode, an initial state $x_1$ is picked arbitrarily by an adversary. Then, at each stage $h$, the agent observes state $x_h\in\mathcal{S}$, picks an action $y_h\in \mathcal{A}$, receives a realized reward $r_h(x_h,y_h)$, and then transitions to the next state $x_{h+1}$, which is determined by $x_h, y_h, D_h$. At the final stage $H$, the episode terminates after the agent takes action $y_H$ and receives reward $r_H$. Then next episode begins. Let $K$ denote the number of episodes, and $T$ denote the length of the horizon: $T=H\times K$, where $H$ is a constant. This is the classic setting of episodic MDP, except that in the one-sided-feedback setting, we have the environment randomness $D_h$, that once realized, can help us determine the reward/transition of any alternative feasible action that ``lies on one side'' of our taken action (Section \ref{onesidedef}). The goal is to maximize the total reward accrued in each episode.

A policy $\pi$ of an agent is a collection of functions $\{\pi_h: \mathcal{S}\rightarrow \mathcal{A}\}_{h\in [H]}$. We use $V_h^{\pi}:\mathcal{S}\rightarrow \mathbb{R}$ to denote the value function at stage $h$ under policy $\pi$, so that $V_h^{\pi}(x)$ gives the expected sum of remaining rewards under policy $\pi$ until the end of the episode, starting from $x_h=x$:
\begin{equation*}
    V_h^{\pi}(x):= \mathbb{E}\Big[\sum_{{h'}=h}^H r_{h'}\big(x_{h'},\pi_{h'}(x_{h'})\big) \Big|x_h=x\Big].
\end{equation*}  
$Q_h^{\pi}:\mathcal{S}\times\mathcal{A}\rightarrow \mathbb{R}$ denotes the Q-value function at stage $h$, so that $Q_h^{\pi}(x,y)$ gives the expected sum of remaining rewards under policy $\pi$ until the end of the episode, starting from $x_h=x, y_h=y$:
\begin{equation*}
    Q_h^{\pi}(x, y):= \mathbb{E}\Big[r_h(x_{h}, y)+\sum_{h'=h+1}^H  r_{h'}\big(x_{h'}, \pi_{h'}(x_{h'})\big)\Big |x_h=x, y_h=y\Big]
\end{equation*} 

Let $\pi^*$ denote an optimal policy in the MDP that gives the optimal value functions $V_h^*(x)=\sup_{\pi} V_h^{\pi}(x)$ for any $x\in\mathcal{S}$ and $h\in [H]$. Recall the Bellman equations:

\noindent\begin{minipage}{.5\linewidth}
\begin{equation*}
  \left\{\begin{array}{l}{V_{h}^{\pi}(x)=Q_{h}^{\pi}\left(x, \pi_{h}(x)\right)} \\ 
  {Q_{h}^{\pi}(x, y):=\mathbb{E}_{x^{\prime}, r_h \sim \mathbb{P}(\cdot | x, y)} \left[r_{h}+V^{\pi}_{h+1}\left(x^{\prime}\right)\right]} \\ 
  {V_{h+1}^{\pi}(x)=0, \quad \forall x \in \mathcal{S}}\end{array}\right.
\end{equation*}
\end{minipage}%
\begin{minipage}{.5\linewidth}
\begin{equation*}
  \left\{\begin{array}{l}{V_{h}^{*}(x)=\min _{y} Q_{h}^{*}(x, y)} \\ 
  {Q_{h}^{*}(x, y):=\mathbb{E}_{x^{\prime}, r_h \sim \mathbb{P}(\cdot | x, y)}\left[r_{h}+ V^*_{h+1}\left(x^{\prime}\right)\right]} \\ 
  {V_{h+1}^{*}(x)=0, \quad \forall x \in \mathcal{S}}\end{array}\right.
\end{equation*}
\end{minipage}

We let $\operatorname{Regret}_{MDP}(K)$ denote the expected cumulative regret against $\pi^*$ on the MDP up to the end of episode $k$. Let $\pi_k$ denote the policy the agent chooses at the beginning of the $k$th episode.
\begin{equation}\operatorname{Regret}_{MDP}(K)=\sum_{k=1}^{K}\left[V_{1}^{*}\left(x_{1}^{k}\right)-V_{1}^{\pi_{k}}\left(x_{1}^{k}\right)\right]\end{equation}
\subsection{One-Sided-Feedback }
\label{onesidedef}
Whenever we take an action $y$ at stage $h$, once the environment randomness $D_h$ is realized, we can learn the rewards/transitions for all the actions that lie on \emph{one side} of $y$, i.e., all $y'\le y$ for the \emph{lower} one-sided feedback setting (or all $y'\ge y$ for the \emph{higher} side). This setting requires that the action space can be embedded in a compact subset of $\mathbb{R}$ (Appendix \ref{hqlappendix}), and that the reward/transition only depend on the action, the time step and the environment randomness, even though the feasible action set depends on the state and is assumed to be an interval $\mathcal A\cap [a,\infty)$ for some $a=a_h(x_h)$. We assume that given $D_h$, the next state $x_{h+1}(\cdot)$ is increasing in $y_h$, and $a_h(\cdot)$ is increasing in $x_h$ for the lower-sided-feedback setting. We assume the optimal value functions are concave. These assumptions seem strong, but are actually widely satisfied in OR/finance problems, such as inventory control (lost-sales model), portfolio management, airline's overbook policy, online auctions, etc.

\subsection{Full-Feedback }
Whenever we take an action at stage $h$, once $D_h$ is realized, we can learn the rewards/transitions for all state-action pairs. This special case does not require the assumptions in Section \ref{onesidedef}. Example problems include inventory control (backlogged model) and portfolio management.



\section{Algorithms}

\begin{algorithm}[h]
\begin{algorithmic}
\STATE Initialization: $Q_h( y)\leftarrow H, \forall (y, h) \in\mathcal{A}\times [H]$; \hspace{0.2cm} $A_h^0\leftarrow\mathcal{A}, \forall h \in [H]$; \hspace{0.2cm}  $A_{H+1}^k\leftarrow\mathcal{A},\forall k\in[K]$;\\
\FOR{$k=1,\ldots,K$}
    \STATE Initiate the list of realized environment randomness to be empty $\mathbb{D}_k=[]$; Receive $x_1^k$;
    \FOR{$h=1,\ldots,H$}
    \IF{$\max\{A_{h}^{k}\} $ is not feasible} 
        \STATE {Take action $y_h^k\leftarrow$ closest feasible action to $A_{h}^{k}$;}
        \ELSE \STATE{Take action $y_h^k\leftarrow\max\{A_{h}^{k}\}$;}
     \ENDIF 
    \STATE Observe realized environment randomness $\tilde{D}_h^k$, append it to $\mathbb{D}_k$;\\
\STATE Update $x_{h+1}^{k}\leftarrow x'_{h+1}(x_{h}^{k},y_h^k, \tilde{D}_h^k$);
    \ENDFOR
    \FOR{$h=H,\ldots,1$}
        \FOR{$y \in A_h^k$}
        \STATE Simulate trajectory $x'_{h+1},\dots, x'_{\tau_h^k(x,y)}$ as if we had chosen $y$ at stage $h$ using $\mathbb{D}_k$ until we find $\tau_h^k(x,y)$, which is the next time we are able to choose from $A_{\tau_h^k(x,y)}^k$;\\
        \STATE Update $Q_{h}(y)\leftarrow (1-\alpha_k)Q_h(y)+\alpha_k[\tilde{r}_{h, \tau_h^k(x,y)}+ V_{h+1}(x'_{h+1}(x_{h}^{k},y_h^k, \tilde{D}_h^k))]$;\\
        \ENDFOR
         \STATE Update $y_h^{k*}\leftarrow \argmax_{y\in A_h^{k}}Q_h(y)$;\\
        \STATE Update $A_h^{k+1}\leftarrow \{y\in A_h^{k}: \big| Q_h(y^{k*}_{h})-Q_h(y)\big|\le \text{Confidence Interval\footnotemark}\}$;\\
        \STATE Update $V_h(x)\leftarrow  \max_{\text{feasible } y \text{ given }x} Q_h(y)$;\\
    \ENDFOR
\ENDFOR
\end{algorithmic}
\caption{Elimination-Based Half-Q-learning (HQL)}
\label{HQLconcave}
\end{algorithm}
\footnotetext{For convenience, we use a ``Confidence Interval'' of $\frac{8}{\sqrt{k-1}}(\sqrt{H^5\iota})$, where $\iota=9\log(AT)$.}

Without loss of generality, we present \emph{HQL} in the \emph{lower}-sided-feedback setting. We define constants $\alpha_k=(H+1)/(H+k), \forall k\in [K]$. We use $\tilde{r}_{h, h'}$ to denote the cumulative reward from stage $h$ to stage $h'$.  We use $x'_{h+1}(x,y,\tilde{D}_h^k)$ to denote the next state given $x$, $y$ and $\tilde{D}_h^k$. By assumptions in Section \ref{onesidedef}, $Q_h(x,y)$ only depends on the $y$ for Algorithm \ref{HQLconcave}, so we simplify the notation to $Q_h(y)$.
\paragraph{Main Idea of Algorithm \ref{HQLconcave}} At any episode $k$, we have a ``running set'' $A_h^k$ of all the actions that are possibly the best action for stage $h$. Whenever we take an action, we update the Q-values for all the actions in $A_h^k$. To maximize the utility of the lower-sided feedback, we always select the largest action in $A_h^k$, letting us observe the most feedback. We might be in a state where we cannot choose from $A_h^k$. Then we take the closest feasible action to $A_h^k$ (the smallest feasible action in the lower-sided-feedback case). By the assumptions in Section \ref{onesidedef}, this is with high probability the optimal action in this state, and we are always able to observe all the rewards and next states for actions in the running set. During episode $k$, we act in real-time and keep track of the realized environment randomness. At the end of the episode, for each $h$, we simulate the trajectories as if we had taken each action in $A_h^k$, and update the corresponding value functions, so as to shrink the running sets.

\begin{algorithm}[h]
\begin{algorithmic}
\STATE Initialization: $Q_h(x, y)\leftarrow H, \forall (x, y, h) \in\mathcal{S}\times \mathcal{A}\times [H]$. 
\FOR{$k=1,\ldots,K$}
    \STATE Receive $x_1^k$;
    \FOR{$h=1,\ldots,H$}
        \STATE Take action $y_h^k\leftarrow\argmax_{\text{feasible }\textbf{y} \text{ given } \textbf{x}_h^k} Q_h(x_h^k,y)$; and observe realized $\tilde{D}_h^k$;\\
        \FOR{$x\in \mathcal{S}$}
        \FOR{$y\in \mathcal{A}$}
        \STATE Update $Q_h(x,y)\leftarrow (1-\alpha_{k})Q_h(x,y)+\alpha_{k}\Big[r_h(x, y,\tilde{D}_h^k ))+ V_{h+1}\big(x'_{h+1}(x, y,\tilde{D}_h^k )\big)\Big];$
        \ENDFOR
        \STATE Update $V_h(x)\leftarrow  \max_{\text{feasible } y \text{ given }x} Q_h(x, y);$\\
        \ENDFOR
        \STATE Update $x_{h+1}^{k}\leftarrow x'_{h+1}(x_{h}^{k},y_h^k, \tilde{D}_h^k)$;
    \ENDFOR
\ENDFOR
\end{algorithmic}
\caption{Full-Q-Learning (FQL)}
\label{FQL}
\end{algorithm}

Algorithm \ref{FQL} is a simpler variant of Algorithm \ref{HQLconcave}, where we effectively set the ``Confidence Interval'' to be always infinity and select the estimated best action instead of maximum of the running set. It can also be viewed as an adaption of \cite{jin2018q} to the full-feedback setting. 


\section{Main Results}
\begin{theorem}
\emph{HQL} has $\mathcal{O}(H^3\sqrt{T\iota})$ total expected regret on the episodic MDP problem in the one-sided-feedback setting. \emph{FQL} has $\mathcal{O}(H^2\sqrt{T\iota})$ total expected regret in the full-feedback setting.
\label{HQLresult}
\end{theorem}

\begin{theorem}
For any (randomized or deterministic) algorithm, there exists a full-feedback episodic MDP problem that has expected regret $\Omega(\sqrt{HT})$, even if the Q-values are independent of the state.
\label{FQLlowerbound}
\end{theorem}

\section{Overview of Proof}
\label{sketches}
 We use $Q_h^k, V_h^k$ to denote the $Q_h, V_h$ functions at the beginning of episode $k$ .
 
Recall $\alpha_{k}=(H+1)/(H+k)$. As in \cite{jin2018q} and \cite{1912.06366}, we define weights $\alpha_{k}^{0}:=\prod_{j=1}^{k}\left(1-\alpha_{j}\right)$, and $ \alpha_{k}^{i}:=\alpha_{i} \prod_{j=i+1}^{k}\left(1-\alpha_{j}\right)$, and provide some useful properties in Lemma \ref{weights}. Note that Property 3 is tighter than the corresponding bound in Lemma 4.1 from \cite{jin2018q}, which we obtain by doing a more careful algebraic analysis.

\begin{lemma}
The following properties hold for $\alpha_t^i$:
\begin{enumerate}
    \item $\sum_{i=1}^{t} \alpha_{t}^{i}=1\text { and } \alpha_{t}^{0}=0 \text { for } t\ge 1$; $\sum_{i=1}^{t} \alpha_{t}^{i}=0 \text { and } \alpha_{t}^{0}=1 \text { for } t=0$.
    \item $\max _{i \in[t]} \alpha_{t}^{i} \leq \frac{2 H}{t} \text { and } \sum_{i=1}^{t}\left(\alpha_{t}^{i}\right)^{2} \leq \frac{2 H}{t} \text { for every } t \geq 1$.
    \item $\sum_{t=i}^{\infty} \alpha_{t}^{i}=1+\frac{1}{H} \text { for every } i \geq 1$.
    \item $\frac{1}{\sqrt{t}} \leq \sum_{i=1}^{t} \frac{\alpha_{t}^{i}}{\sqrt{i}} \leq \frac{1+\frac{1}{H}}{\sqrt{t}} \text { for every } t \geq 1$.
\end{enumerate}
\label{weights}
\end{lemma}


All missing proofs for the lemmas in this section are in Appendix \ref{hqlappendix}.

\begin{lemma}
(shortfall decomposition) 
For any policy $\pi$ and any $k\in [K]$, the regret in episode $k$ is:
\begin{equation}
\Big(V_1^{*}-V_1^{\pi_k}\Big)(x_1^k) =\mathbb{E}_{\pi}\Big[\sum_{h=1}^{H}\big(\max _{y \in \mathcal{A}} Q^{*}_h\left(x_h^k, y\right)-Q^{*}_h\left( x_h^k, y_h^{k}\right)\big) \Big].
\end{equation}
\label{shortfall}
\end{lemma}
\vspace{-0.2cm}
Shortfall decomposition lets us calculate the regret of our policy by summing up the difference between Q-values of the action taken at each step by our policy and of the action the optimal $\pi^*$ would have taken if it was in the same state as us. We need to then take expectation of this random sum, but we get around this by finding high-probability upper-bounds on the random sum as follows:



Recall for any $(x, h, k)\in \mathcal{S} \times [H]\times [K]$, and for any $y\in A_h^k$, $\tau_h^k(x, y)$ is the next time stage after $h$ in episode $k$ that our policy lands on a simulated next state ${x_{\tau_h^k(x, y)}^k}'$ that allows us to take an action in the running set $A_{\tau_h^k(x, y)}^k$. The time steps in between are ``skipped'' in the sense that we do not perform Q-value updating or V-value updating during those time steps when we take $y$ at time $(h,k)$. Over all the $h'\in[H]$, we only update Q-values and V-values while it is feasible to choose from the running set. E.g.\ if no skipping happened, then $\tau_h^k(x, y)=h+1$. Therefore, $\tau_h^k(x, y)$ is a stopping time. Using the general property of optional stopping that $\mathbb E[M_{\tau}]=M_0$ for any stopping time $\tau$ and discrete-time martingale $M_{\tau}$, our Bellman equation becomes
\begin{equation}
    Q_h^*(y)=\mathbb{E}_{\tilde{r}^*_{h, \tau_h^k}, x'_{\tau_h^k}, \tau_h^k\sim\mathbb{P}(\cdot|x,y)}[\tilde{r}^*_{h, \tau_h^k} + V^*_{\tau_h^k}(x'_{\tau_h^k})]
    \label{Bellmanskip}
\end{equation}
where we simplify notation $\tau_h^k(x,y)$ to $\tau_h^k$ when there is no confusion, and recall $\tilde{r}_{h,h'}$ denotes the cumulative reward from stage $h$ to $h'$. 
On the other hand, by simulating paths, \emph{HQL} updates the $Q$ functions backward $h=H, \dots, 1$ for any $x\in \mathcal{S}$, $y\in A_h^k$ at any stage $h$ in any episode $k$ as follows:
\begin{equation}
    Q_h^{k+1}(y)\leftarrow (1-\alpha_{k})Q_h^{k}(y)+\alpha_{k}[{\tilde{r}_{\tau_h^{k+1}(x,y)}^{k+1}(x,y)}+ V_{\tau_h^{k+1}(x,y)}^{k+1}(x'_{\tau_h^{k+1}(x,y)})]
    \label{updateHQL}
\end{equation}
Then by Equation \ref{updateHQL} and the definition of $\alpha_k^i$'s, we have
\begin{equation}Q_{h}^{k}(y)=\alpha_{k-1}^{0} H+\sum_{i=1}^{k-1} \alpha_{k-1}^{i}\left[\tilde{r}_{h,\tau_h^k(x,y)}^{i}+V_{\tau_h^k(x,y)}^{i+1}\left(x_{\tau_h^k(x,y)}^{i}\right)\right].
    \label{Qupdatehql}
\end{equation}
which naturally gives us Lemma \ref{Qdiff}. For simpler notation, we use $\tau_h^i=\tau_h^i(x,y)$.
\begin{lemma}
For any $(x, h, k)\in \mathcal{S} \times [H]\times [K]$, and for any $y\in A_h^k$, we have
\begin{equation*}
    \begin{aligned}
    \left(Q_{h}^{k}-Q_{h}^{*}\right)(y)=&\alpha_{k-1}^{0}\left(H-Q_{h}^{\star}(y)\right)+\sum_{i=1}^{k-1} \alpha_{k-1}^{i}\Big[\left(V_{\tau_h^i}^{i+1}-V_{\tau_h^i}^{*}\right)(x_{\tau_h^i}^{i})+\tilde{r}^i_{h, \tau_h^i}\\
    &-\tilde{r}^*_{h, \tau_h^i}+\left(V^*_{\tau_h^i(x, y)}(x_{\tau_h^i}^i)+\tilde{r}^*_{h, \tau_h^i}-\mathbb{E}_{\tilde{r}^*,x', \tau_h^i\sim\mathbb{P}(\cdot|x,y)}\big[\tilde{r}^*_{h,\tau_h^i} + V^*_{\tau_h^i}(x'_{\tau_h^i})\big]\right)\Big].
    \end{aligned}
\end{equation*}

\label{Qdiff}
\end{lemma}

Then we can bound the difference between our Q-value estimates and the optimal Q-values:

\begin{lemma}
For any $(x, h, k)\in \mathcal{S} \times [H]\times [K]$, and any $y\in A_h^k$, let $\iota=9\log(AT)$, we have:
\[\Big|\left(Q_{h}^{k}-Q_{h}^{*}\right)( y)\Big| \le \alpha_{k-1}^{0} H+ \sum_{i=1}^{k-1} \alpha_{k-1}^{i}\Big|\big(V_{\tau_h^i}^{i+1}-V_{\tau_h^i}^{*}\big)\big(x_{\tau_h^i}^{i}\big)+\tilde{r}^i_{h,\tau_h^i}-\tilde{r}^*_{h,\tau_h^i}\Big|+c \sqrt{\frac{H^{3} \iota}{k-1}}\]

with probability at least $1-1/(AT)^8$, and we can choose $c=2\sqrt{2}$.
\label{martingale}
\end{lemma}
Now we define $\{\delta_h\}_{h=1}^{H+1}$ to be a list of values that satisfy the recursive relationship
\[\delta_h=H+(1+1/H)\delta_{h+1}+c \sqrt{H^{3} \iota} \text{ , for any $h\in [H]$},\]
where $c$ is the same constant as in Lemma \ref{martingale}, and $\delta_{H+1}=0$. 
Now by Lemma \ref{martingale}, we get:
\begin{lemma}
For any $( h, k)\in [H]\times [K]$,
$\{\delta_h\}_{h=1}^H$ is a sequence of values that satisfy
\begin{equation*}
    \hspace{1cm}\max_{y\in A_h^k}\Big|(Q_h^k-Q_h^*)(y)\Big|\le \delta_h/\sqrt{k-1} \hspace{1cm}\text{with probability at least } 1-1/(AT)^5.
    \label{inequalityofdelta}
\end{equation*} 
\label{delta}
\end{lemma}

\vspace{-0.2cm}
Lemma \ref{delta} helps the following three lemmas show the validity of the running sets $A_h^k$'s:

\begin{lemma}
For any $h\in [H], k\in[K]$, the optimal action $y_h^*$ is in the running set $A_h^k$ with probability at least $1-1/(AT)^5$.
\label{insideOPT}
\end{lemma}

\begin{lemma}
Anytime we can play in $A_h^k$, the optimal Q-value of our action is within $3\delta_h/\sqrt{k-1}$ of the optimal Q-value of the optimal policy's action, with probability at least $1-2/(AT)^5$. 
\label{within}
\end{lemma}

\begin{lemma}
Anytime we cannot play in $A_h^k$, our action that is the feasible action closest to the running set is the optimal action for the state $x$ with probability at least $1-1/(AT)^5$.
\label{outside}
\end{lemma}

Naturally, we want to partition the stages $h=1,\dots, H$ in each episode $k$ into two sets, $\Gamma_A^k$ and $\Gamma_B^k$, where $\Gamma_A^k$ contains all the stages $h$ where we are able to choose from the running set, and $\Gamma_B^k$ contains all the stages $h$ where we are unable to choose from the running set. So $\Gamma_B^k \sqcup \Gamma_A^k=[H], \forall k\in[K]$.

Now we can prove Theorem \ref{HQLresult}. By Lemma \ref{shortfall} we have that 
\begin{equation*}
\begin{aligned}
&V_h^{*}-V_h^{\pi_k} =\mathbb{E}\Big[\sum_{h=1}^{H}\left(\max _{y \in \mathcal{A}} Q^{*}_h\left( y\right)-Q^{*}_h\left( y_h^{k}\right)\right) \Big]\le\mathbb{E}\Big[\sum_{h=1}^{H}\max _{y \in \mathcal{A}} \Big(Q^{*}_h\left( y\right)-Q^{*}_h\left( y_h^{k}\right)\Big) \Big]\\
&\le\mathbb{E}\Big[\sum_{h\in \Gamma_A^k}\max _{y \in \mathcal{A}} \Big(Q^{*}_h\left( y\right)-Q^{*}_h\left( y_h^{k}\right)\Big)\Big]+\mathbb{E}\Big[\sum_{h\in \Gamma_B^k}\max _{y \in \mathcal{A}} \Big(Q^{*}_h\left( y\right)-Q^{*}_h\left( y_h^{k}\right)\Big) \Big].\\
\end{aligned}
\end{equation*}
By Lemma \ref{outside}, the second term is upper bounded by
\begin{equation}
    0\cdot(1-\frac{1}{A^5T^5})+\sum_{h\in \Gamma_B^k}H\cdot\frac{1}{A^5T^5}\le \sum_{h\in \Gamma_B^k} H\cdot \frac{1}{A^5T^5}.
\end{equation}
By Lemma \ref{delta}, the first term is upper-bounded by
\begin{equation*}
\begin{aligned}
&\mathbb{E}\bigg[\sum_{h\in \Gamma_A^k}\mathcal{O}\Big(\frac{\delta_h}{\sqrt{k-1}}\Big)\bigg]\mathbb{P}\Big( \max _{y \in A_h^k} \Big(Q^{*}_h\left( y\right)-Q^{*}_h\left( y_h^{k}\right)\Big)\le\frac{\delta_h}{\sqrt{k-1}} \Big)\\
&+\sum_{h\in \Gamma_A^k}H\cdot\mathbb{P}\Big( \max _{y \in A_h^k} \Big(Q^{*}_h\left( y\right)-Q^{*}_h\left( y_h^{k}\right)\Big)>\frac{\delta_h}{\sqrt{k-1}} \Big)\le \mathcal{O}\Big(\sum_{\sum_{h\in \Gamma_A^k}}\frac{\delta_h}{\sqrt{k-1}}\Big)+\mathcal{O}\Big(\sum_{\sum_{h\in \Gamma_A^k}}\frac{H}{A^5T^5}\Big).\\
\end{aligned}
\end{equation*}
Then the expected cumulative regret between \emph{HQL} and the optimal policy is:
\begin{equation*}
\begin{aligned}
&\operatorname{Regret}_{MDP}(K)=\sum_{k=1}^{K}\left(V_{1}^{*}-V_{1}^{\pi_{k}}\right)\left(x_{1}^{k}\right)= \left(V_{1}^{*}-V_{1}^{\pi_{1}}\right)\left(x_{1}^{1}\right)+\sum_{k=2}^{K}\left(V_{1}^{*}-V_{1}^{\pi_{k}}\right)\left(x_{1}^{k}\right) \\
&\le H+ \sum_{k=2}^K\Big(\sum_{h\in \Gamma_B^k}^H  \frac{H}{A^5T^5}+\sum_{\sum_{h\in \Gamma_A^k}}\frac{\delta_h}{\sqrt{k-1}}+\sum_{\sum_{h\in \Gamma_A^k}}\frac{H}{A^5T^5} \Big)\le \sum_{k=2}^K\frac{\mathcal{O}(\sqrt{H^7\iota})}{\sqrt{k-1}}\le \mathcal{O}(H^{3}\sqrt{T\iota}).\hspace{1em}\square
\end{aligned}
\end{equation*}

\subsection{Proofs for \emph{FQL} }
\label{prooffql}

Our proof for \emph{HQL} can be conveniently adapted to recover the same regret bound for \emph{FQL} in the full-feedback setting. We need a variant of Lemma \ref{within}: whenever we take the estimated best feasible action in \emph{FQL}, the optimal Q-value of our action is within $\frac{3\delta_h}{\sqrt{k-1}}$ of the optimal Q-value of the optimal action, with probability at least $1-2/(AT)^5$. Then using Lemmas \ref{shortfall},\ref{Qdiff},\ref{martingale} and \ref{insideOPT} where all the $Q_h^k(y)$ are replaced by $Q_h^k(x,y)$, the rest of the proof follows without needing the assumptions for the one-sided-feedback setting.

For the tighter $\mathcal{O}(H^{2}\sqrt{T\iota})$ regret bound for \emph{FQL} in Theorem \ref{HQLresult}, we adopt similar notations and proof in \cite{jin2018q} (but adapted to the full-feedback setting) to facilitate quick comprehension for readers who are familiar with \cite{jin2018q}. The idea is to use $\left(V_{1}^{k}-V_{1}^{\pi_{k}}\right)\left(x_{h}^{k}\right)$ as a high probability upper-bound on $\left(V_{1}^{*}-V_{1}^{\pi_{k}}\right)\left(x_{1}^{k}\right)$,
and then upper-bound it using martingale properties and recursion. Because \emph{FQL} leverages the full feedback, it shrinks the concentration bounds much faster than existing algorithms, resulting in a significantly lower regret bound. See Appendix \ref{FQLappendix}.

\section{Example Applications: Inventory Control and More}
\label{inventorysection}
 \subsection{Inventory Control} 
 Inventory control is one of the most fundamental problems in supply chain optimization. It is known that base-stock policies (aka. order-up-to policies) are optimal for the classical models we are concerned with (\cite{zipkin2000foundations}, \cite{textbook}). Therefore, we let the actions for the episodic MDP be the amounts to order inventory up to. At the beginning of each step $h$, the retailer sees the inventory $x_h\in \mathbb{R}$ and places an order to raise the inventory level up to $y_h\ge x_h$. Without loss of generality, we assume the purchasing cost is $0$ (Appendix \ref{invappendix}). Replenishment of $y_h-x_h$ units arrive instantly. Then an independently distributed random demand $D_h$ from unknown distribution $F_h$ is realized. We use the replenished inventory $y_h$ to satisfy demand $D_h$. At the end of stage $h$, if demand $D_h$ is less than the inventory, what remains becomes the starting inventory for the next time period $x_{h+1}=(y_h-D_h)^+$, and we pay a holding cost $o_h$ for each unit of left-over inventory. 
 
 {\bf Backlogged model}: if demand $D_h$ exceeds the inventory, the additional demand is backlogged, so the starting inventory for the next period is $x_{h+1}=y_h-D_h<0$. We pay a backlogging cost $b_h>0$ for each unit of the extra demand.  The reward for period $h$ is the negative cost:
\begin{equation*}
  r_h(x_h,y_h)=-\big(c_h(y_h-x_h)+o_h(y_h-D_h)^++b_h(D_h-y_h)^+\big). 
\end{equation*}
This model has full feedback because once the environment randomness--the demand is realized, we can deduce what the reward and leftover inventory would be for all possible state-action pairs.
 
{\bf Lost-sales model}: is considered more difficult. When the demand exceeds the inventory, the extra demand is lost and unobserved instead of backlogged. We pay a penalty of $p_h>0$ for each unit of lost demand, so the starting inventory for next time period is $x_{h+1}=0$. The reward for period $h$ is:
\begin{equation*}r_h(x_h,y_h)=-\big(c_h(y_h-x_h)+o_h(y_h-D_h)^++p_h(D_h-y_h)^+\big).\end{equation*}
Note that we cannot observe the realized reward because the extra demand $(D_h-y_h)^+$ is unobserved for the lost-sales model. However, we can use a pseudo-reward $r_h(x_h, y_h)=-\big(o_h(y_h-D_h)^+-p_h\min(y_h,D_h)\big)$ that will leave the regret of any policy against the optimal policy unchanged (\cite{agrawal2019learning}, \cite{yuan2019marrying}). This pseudo-reward can be observed because we can always observe $\min(y_h,D_h)$. Then this model has (lower) one-sided feedback because once the environment randomness--the demand is realized, we can deduce what the reward and leftover inventory would be for all possible state-action pairs where the action (order-up-to level) is lower than our chosen action, as we can also observe $\min(y'_h,D_h)$ for all $y'_h\le y_h$.

\paragraph{Past literature} typically studies under the assumption that the demands along the horizon are i.i.d. (\cite{agrawal2019learning}, \cite{zhang2018closing}). Unprecedentedly, our algorithms solve optimally the episodic version of the problem where the demand distributions are arbitrary within each episode.

\paragraph{Our result}: it is easy to see that for both backlogged and lost-sales models, the reward only depends on the action, the time step and the realized demand, not on the state--the starting inventory. However, the feasibility of an action depends on the state, because we can only order up to a quantity no lower than the starting inventory. The feasible action set at any time is $\mathcal A\cap [x_h,\infty)$. The next state $x_{h+1}(\cdot)$ and $a_h(\cdot)$ are monotonely non-decreasing, and the optimal value functions are concave.
 
Since inventory control literature typically considers a continuous action space $[0, M]$ for some $M\in \mathbb{R}^+$, we discretize $[0, M]$ with step-size $\frac{M}{T^2}$, so $A=|\mathcal{A}|=T^2$. Discretization incurs additional regret $\operatorname{Regret}_{gap}= \mathcal{O}(\frac{M}{T^2}\cdot HT)=o(1)$ by Lipschitzness of the reward function. 
For the lost-sales model, \emph{HQL} gives $\mathcal{O}(H^3\sqrt{T\log T})$ regret. For the backlogged model, \emph{FQL} gives $\mathcal{O}(H^2\sqrt{T\log T})$ regret, and \emph{HQL} gives $\mathcal{O}(H^3\sqrt{T\log T})$ regret. See details in Appendix \ref{invappendix}.

\paragraph{Comparison with existing Q-learning algorithms}: If we discretize the state-action space optimally for \cite{jin2018q} and for \cite{1912.06366}, then applying \cite{jin2018q} to the backlogged model gives a regret bound of $\mathcal{O}(T^{3/4}\sqrt{\log T})$. Applying \cite{1912.06366} to the backlogged inventory model with optimized aggregation gives us $\mathcal{O}(T^{2/3}\sqrt{\log T})$. See details in Appendix \ref{priorQ}.

\subsection{Other Example Applications}
\paragraph{Online Second-Price Auctions}: the auctioneer needs to decide the reserve price for the same item at each round (\cite{bandit}). Each bidder draws a value from its unknown distribution and only submits the bid if the value is no lower than the reserve price. The auctioneer observes the bids, gives the item to the highest bidder if any, and collects the second highest bid price (including the reserve price) as profits. In the episodic version, the bidders' distributions can vary with time in an episode, and the horizon consists of $K$ episodes. This is a (higher) one-sided-feedback problem that can be solved efficiently by \emph{HQL} because once the bids are submitted, the auctioneer can deduce what bids it would have received for any reserve price higher than the announced reserve price. 

\paragraph{ Airline Overbook Policy}: is to decide how many customers the airline allows to overbook a flight (\cite{airline}). This problem has lower-sided feedback because once the overbook limit is reached, extra customers are unobserved, similar to the lost-sales inventory control problem.

\paragraph{Portfolio Management} is allocation of a fixed sum of cash on a variety of financial instruments (\cite{nobel}). In the episodic version, the return distributions are episodic. On each day, the manager collects the increase in the portfolio value as the reward, and gets penalized for the decrease. This is a full-feedback problem, because once the returns of all instruments become realized for that day, the manager can deduce what his reward would have been for all feasible portfolios.

\section{Numerical Experiments}
\label{numerical}
We compare \emph{FQL} and \emph{HQL} on the backlogged episodic inventory control problem against 3 benchmarks: the optimal policy (\emph{OPT}) that knows the demand distributions beforehand and minimizes the cost in expectation, \emph{QL-UCB} from \cite{jin2018q}, and \emph{Aggregated QL} from \cite{1912.06366}. 

For \emph{Aggregated QL} and \emph{QL-UCB}, we optimize by taking the Q-values to be only dependent on the action, thus reducing the state-action pair space. \emph{Aggregated QL} requires a good aggregation of the state-action pairs to be known beforehand, which is usually unavailable for online problems. We aggregate the state and actions to be multiples of $1$ for \cite{1912.06366} in Table \ref{alltable}. We do not fine-tune the confidence interval in \emph{HQL}, but use a general formula $\sqrt{\frac{H\log(HKA)}{k}}$ for all settings. We do not fine-tune the UCB-bonus in \emph{QL-UCB} either. Below is a summary list for the experiment settings. Each experimental point is run $300$ times for statistical significance.

\noindent\begin{minipage}{.5\linewidth}
 {\bf Episode length}: $H=1, 3, 5$.
 
 {\bf Number of episodes}: $K=100,500,2000$.
 
{\bf Demands}: $D_h\sim (10-h)/2+U[0,1]$.

\end{minipage}%
\begin{minipage}{.5\linewidth}
 {\bf Holding cost}: $o_h=2$.
 
 {\bf Backlogging cost}: $b_h=10$.
 
{\bf Action space}: $[0, \frac{1}{20}, \frac{2}{20}, \dots, 10]$.
\end{minipage}
\vspace{-0.8cm}

\begin{center}
\begin{table}[ht]
 \caption{Comparison of cumulative costs for backlogged episodic inventory control}
 \begin{tabular}{p{0.2cm}p{0.8cm}p{0.8cm}p{0.6cm}p{0.8cm}p{0.6cm}p{0.8cm}p{0.7cm}p{1cm}p{0.9cm}p{1cm}p{0.6cm}} 
\multicolumn{2}{c}{} & \multicolumn{2}{c}{ \emph{OPT} } &  \multicolumn{2}{c}{\emph{FQL} } & \multicolumn{2}{c}{\emph{HQL} } & \multicolumn{2}{c}{ \emph{Aggregated QL}} & \multicolumn{2}{c}{ \emph{QL-UCB}}\\  
{\bf H} & {\bf K }  & {\bf mean} & {\bf SD} &  {\bf mean} & {\bf SD} &  {\bf mean} & {\bf SD}  &  {\bf mean} & {\bf SD} &  {\bf mean} & {\bf SD} \\  
 \hline

 \multirow{3}{*}{1} & $100 $ & 88.2 & 4.1 & 103.4 & 6.6  & 125.9 & 19.2 & 406.6 & 16.1 & 3048.7 & 45.0\\ 
  & $500 $&  437.2   & 4.4 & 453.1 & 6.6  & 528.9 & 44.1 & 1088.0 & 62.2 & 4126.3 & 43.7 \\ 
 & $2000 $&  1688.9  & 2.8 &  1709.5 & 5.8  & 1929.2 & 89.1  & 2789.1 & 88.3 & 7289.5 & 57.4\\ 
 \multirow{3}{*}{3} & $100$ & 257.4 & 3.2  & 313.1 & 7.6  & 435.1 & 17.6 & 867.9 & 29.2 & 7611.1 & 46.7\\
 & $500$& 1274.6 & 6.1 & 1336.3 & 10.5  & 1660.2 & 48.7  & 2309.1 & 129.8 & 10984.0 & 73.0 \\
 & $2000$ & 4965.6 & 8.3  & 5048.2 & 13.3 & 5700.6 & 129.1 & 7793.5  & 415.6  & 22914.7  & 131.1\\
 \multirow{3}{*}{5} & $ 100$ & 421.2 & 3.3 & 528.0 & 10.4  & 752.6 & 32.9 & 1766.8 & 83.8 & 11238.4 & 140.0 \\ 
& $500$  &  2079.0 & 8.2  & 2204.0 & 13.1  &  2735.1  & 114.1 & 4317.5 & 95.8 & 15458.1 & 231.8  \\ 
& $ 2000$ &  8285.7 & 8.3 & 8444.7 & 16.4 &  9514.4 &  364.2 & 13373.0  & 189.2 & 40347.0 & 274.6\\ 
\end{tabular}
\vspace{-0.5cm}
 \label{alltable}
\end{table}
\end{center}
\vspace{-0.3cm}
Table \ref{alltable} shows that both \emph{FQL} and \emph{HQL} perform promisingly, with significant advantage over the other two algorithms. \emph{FQL} stays consistently very close to the clairvoyant optimal, while \emph{HQL} catches up rather quickly using only one-sided feedback. See more experiments in Appendix \ref{numericappend}.
 \section{Conclusion}
 \label{conclusion}
 We propose a new Q-learning based framework for reinforcement learning problems with richer feedback. Our algorithms have only logarithmic dependence on the state-action space size, and hence are barely hampered by even infinitely large state-action sets. This gives us not only efficiency, but also more flexibility in formulating the MDP to solve a problem. Consequently, we obtain the first $\mathcal{O}(\sqrt T)$ regret algorithms for episodic inventory control problems. We consider this work to be a proof-of-concept showing the potential for adapting reinforcement learning techniques to problems with a broader range of structures.



\newpage
\bibliography{main}
\bibliographystyle{alpha}

\newpage
\appendix
\begin{center}
    {\LARGE{Appendices} }
\end{center}

\section{Proof for Lower Bounds (Theorem \ref{FQLlowerbound})}
\label{lowerboundsappend}
We construct here an easy instance of the episodic inventory control problem (as in Section \ref{inventorysection}), for which the regret of any algorithm must be at least $\Omega(\sqrt{HT})$. 

\begin{proof}
Suppose for any time $h$ in an episode, the demand distribution is $h+100$ units w.p. $0.5+\frac{1}{\sqrt{K}}$, and $h+200$ units w.p. $0.5-\frac{1}{\sqrt{K}}$. Note that this is not a constant gap, because $K=\Theta(T)$. Suppose the unit costs for holding, backlogging, and lost-sales penalty are all the same. We generously allow the algorithm to have the correct prior that the best base stock level is one of these two actions, and the other actions are worse than these two actions. Then for each time step $h$, our problem of estimating the $Q$-values degenerates to the stochastic full-feedback online bandit problem. 

It is a well-known result that in this case, each stage $h$ will incur at least a $\Omega(\sqrt{K})$ regret across the $K$ episodes. In particularly, at any time step of any episode, the probability of any algorithm choosing the wrong action is lower-bounded by $\frac{1}{12}$: see Corollary 2.10 in \cite{slivkins2019introduction}. Then at each time step, the algorithm incur a $\Omega(\frac{1}{12\sqrt{K}})$ expected regret. This regret at stage $h$ across the $K$ episodes sum up to $\Omega(\sqrt{K})$ expected regret. Since there are $H$ time steps with demand independent from each other, we have that the regret of this example is lower bounded by $\Omega(H\sqrt{K})=\Omega(\sqrt{HT})$ regret. Note that even though we assume the algorithm receives full information feedback at each time step, Corollary 2.10 in \cite{slivkins2019introduction} still applies by scaling the time horizon by a factor of $2$, which does not affect the regret bound. Then we put back the $\Theta(M\cdot\max(|o_h|,|b_h|))$ factor (or $\Theta(M\cdot\max(|o_h|,|p_h|))$ factor) because in the Preliminaries we scaled the unit costs down by $\Theta(M)$ to have the reward for each time period bounded by $1$.  
\end{proof}



\section{Missing Proofs for \emph{HQL}}
\label{hqlappendix}

\begin{proof}{\bf (Lemma \ref{weights})}
We prove number (4) by induction. For the base case $t=1$, we have $\sum_{i=1}^t\frac{\alpha_t^i}{\sqrt{i}}=\alpha_1^1=1$ so the statement holds. For $t\ge 2$, by the relationship $\alpha_t^i=(1-\alpha_t)\alpha_{t-1}^i$ for $i=1, \dots, t-1$ we have 
\begin{equation}
\sum_{i=1}^{t} \frac{\alpha_{t}^{i}}{\sqrt{i}}=\frac{\alpha_{t}}{\sqrt{t}}+\left(1-\alpha_{t}\right) \sum_{i=1}^{t-1} \frac{\alpha_{t-1}^{i}}{\sqrt{i}}
\end{equation}
Assuming the inductive hypothesis holds, on the one hand, 
\begin{equation}
\frac{\alpha_{t}}{\sqrt{t}}+\left(1-\alpha_{t}\right) \sum_{i=1}^{t-1} \frac{\alpha_{t-1}^{i}}{\sqrt{i}} \geq \frac{\alpha_{t}}{\sqrt{t}}+\frac{1-\alpha_{t}}{\sqrt{t-1}} \geq \frac{\alpha_{t}}{\sqrt{t}}+\frac{1-\alpha_{t}}{\sqrt{t}}=\frac{1}{\sqrt{t}}
\end{equation}
where the first inequality holds by the inductive hypothesis.
On the other hand,
\begin{equation}
\begin{aligned}
\frac{\alpha_{t}}{\sqrt{t}}+\left(1-\alpha_{t}\right) \sum_{i=1}^{t-1} \frac{\alpha_{t-1}^{i}}{\sqrt{i}} & \leq \frac{\alpha_{t}}{\sqrt{t}}+\frac{(1+1/H)\left(1-\alpha_{t}\right)}{\sqrt{t-1}}=\frac{H+1}{\sqrt{t}(H+t)}+\frac{(1+1/H) \sqrt{t-1}}{H+t} \\
& \leq \frac{H+1}{\sqrt{t}(H+t)}+\frac{(1+1/H) \sqrt{t}}{H+t}\le \frac{(1+1/H)}{\sqrt{t}}
\end{aligned}
\end{equation}
where the first inequality holds by the inductive hypothesis.
\end{proof}
This is a tighter bound than the bound in \cite{jin2018q}. For rest of the lemma, see Lemma 4.1 in \cite{jin2018q}. 

The following proof for shortfall decomposition is adapted from Benjamin Van Roy's reinforcement learning notes for the class MS 338 at Stanford University. 
\begin{proof}{\bf (Lemma \ref{shortfall})}
For any policy $\pi$, let $y_h^k$ denote the action the policy $\pi_k$ takes at stage $h$ of episode $k$. Let $R_h$ denote the expected reward of $y_h^k$.

\noindent\begin{minipage}{.4\linewidth}
\begin{equation*}
\mathbb{E}_{\pi}\left[Q^{*}\left(x_{h}^k, y_{h}^k\right)\right]=\mathbb{E}_{\pi}\left[Z_{h+1}\right] 
\end{equation*}
\end{minipage}%
\begin{minipage}{.6\linewidth}
\begin{equation*}
\text{where }Z_{h+1}=\left\{\begin{array}{ll}
R_{h}+\max _{y} Q^{*}\left(x_{h+1}^k, y\right) & \text { if } h<H \\
R_{h} & \text { if } h=H
\end{array}\right.
\end{equation*}
\end{minipage}
\vspace{0.5cm}

Therefore,
\begin{equation*}
\begin{aligned}
 V_1^{*}-V_1^{\pi_k} &=\mathbb{E}_{\pi}\Big[\max _{a \in \mathcal{A}} Q^{*}(x_{1}^k, a)-\sum_{h=1}^{H} R_{h} \Big]\\
 &=\mathbb{E}_{\pi}\Big[\max _{a \in \mathcal{A}} Q^{*}\left(x_{1}^k, a\right)-\sum_{h=1}^{H}\left(R_{h}-Z_{h+1}+Q^{*}\left(x_{h}^k, y_{h}^k\right)\right)\Big] \\
&=\mathbb{E}_{\pi}\Big[\sum_{h=1}^{H}\big(\max _{a \in \mathcal{A}} Q^{*}(x_{h}^k, a)-Q^{*}(x_{h}, y^k_{h})\big) \Big] 
\end{aligned}
\end{equation*}
\end{proof}

\begin{proof}{\bf (Lemma \ref{Qdiff})}
From the Bellman optimality equation (\ref{Bellmanskip}), and the fact that $ \sum_{i=0}^{k-1} \alpha_{k-1}^{i}=1$, we have
\[Q_{h}^{*}(y)=\alpha_{k-1}^{0} Q_{h}^{*}(y)+\sum_{i=1}^{k-1} \alpha_{k-1}^{i}\left[\mathbb{E}_{x', \tau_h^i(y)\sim\mathbb{P}(\cdot|x,y)}[\tilde{r}^*_{\tau_h^i(x, y)}(y) + V^*_{\tau_h^i(x, y)}(x'_{\tau_h^i( y)})]\right]\]

Subtracting Equation \ref{Qupdatehql} from this equation, and adding some of the middle terms that cancel with themselves gives us Lemma \ref{Qdiff}.
\end{proof}

\begin{proof}{\bf (Lemma \ref{martingale})} Since we assume that given a fixed value $D_h$, the next state $x_{h+1}(y_h)$ is increasing in $y_h$, and $a_h(x_h)$ is increasing in $x_h$ for the lower one-sided-feedback problem, we conclude that the (deterministic given $D_h$) dynamics are monotone with respect to any simulation starting point $x_h$. Since the algorithm chooses at least the maximal action in $A^k_h$ at all times, this implies it can observe the simulated trajectory started from any $x_h\in A^k_h$ for any $k,h\in [K]\times[H]$. 


Let $\mathcal{F}_h^i$ be the $\sigma$-field generated by all the random variables until episode $i$, stage $h$. Then for any $\tau\in[K]$, $\left(V^*_{\tau_h^i(x, y)}(x_{\tau_h^i(x, y)}^i)+\tilde{r}^*_{\tau_h^i(x, y)}-\mathbb{E}_{\tilde{r}^*,x', \tau_h^i(x, y)\sim\mathbb{P}(\cdot|x,y)}\Big[\tilde{r}^*_{\tau_h^i(x, y)} + V^*_{\tau_h^i(x, y)}(x'_{\tau_h^i(x,y)})\Big]\right)_{i=1}^{\tau}$ is a martingale difference sequence w.r.t. the filtration $\left\{\mathcal{F}_h^{i}\right\}_{i \geq 0}$. Then by Azuma-Hoeffding Theorem, we have that with probability at least $1-(1/AT)^9$:
\begin{equation}
\begin{aligned}
&\left|\sum_{i=1}^{k-1} \alpha_{k-1}^{i}\cdot\left(V^*_{\tau_h^i(x, y)}(x_{\tau_h^i(x, y)}^i)+\tilde{r}^*_{\tau_h^i(x, y)}-\mathbb{E}_{\tilde{r}^*,x', \tau_h^i(x, y)\sim\mathbb{P}(\cdot|x,y)}\Big[\tilde{r}^*_{\tau_h^i(x, y)} + V^*_{\tau_h^i(x, y)}(x'_{\tau_h^i(x, y)})\Big]\right)\right| \\
&\leq \frac{c H}{2} \sqrt{\sum_{i=1}^{k-1}\left(\alpha_{k-1}^{i}\right)^{2} \cdot \iota} \hspace{0.1cm}\leq c \sqrt{\frac{H^{3} \iota}{k-1}} \hspace{0.5cm}\text{ for any constant $c\ge 2\sqrt{2}$.}
\end{aligned}
    \label{azuma}
\end{equation}
By union bound, we have with probability at least $1-(1/AT)^8$ that for any $x, h,k, y\in A_h^k$,
\begin{equation*}
\begin{split}
     \left|\sum_{i=1}^{k-1} \alpha_{k-1}^{i}\left(V^*_{\tau_h^i(x, y)}(x_{\tau_h^i(x, y)}^i)+\tilde{r}^*_{\tau_h^i(x, y)}-\mathbb{E}_{\tilde{r}^*, x', \tau_h^i(x, y)\sim\mathbb{P}(\cdot|x,y)}\Big[\tilde{r}^*_{\tau_h^i(x, y)} + V^*_{\tau_h^i(x, y)}(x'_{\tau_h^i(x,  y)})\Big]\right)\right|\\
     \leq c \sqrt{\frac{H^{3} \iota}{k-1}}    
\end{split}
    \label{unionbound}
\end{equation*}
Then Lemma \ref{martingale} follows immediately this equation and Lemma \ref{Qdiff}. 
\end{proof}

\begin{proof}{\bf (Upper Bound on $\delta_h$'s)}
We set $d_h=(\delta_h)\cdot\left(1+\frac{1}{H}\right)^{h}$ and observe that the recurrence implies
\begin{equation}d_h= d_{h+1}+H+2\sqrt{2}\sqrt{H^3\iota}\end{equation}

Then from this recursion we see $d_h\le H^2+2\sqrt{2H^5\iota}$ for all $h$. Since $d_h,\delta_h$ differ by a constant factor $\left(1+\frac{1}{H}\right)^{h}$, we have $\delta_h= \frac{H^2+2\sqrt{2H^5\iota}}{\left(1+\frac{1}{H}\right)^{h}}\le 4\sqrt{H^5\iota}$.
\end{proof}

\begin{proof}{\bf (Lemma \ref{delta})}
We prove by backward induction. Note that all of our statements below hold with high probability. In particular, we will use Azuma-Hoeffding no more than $AT$ times in the below, with each use holding with probability at least $1/(AT)^9$. Under the assumption that each use of Azuma-Hoeffding holds we will obtain the statement of the Lemma. Our proof goes by induction; for the base case $\delta_{H+1}=0$ satisfies the Inequality \ref{inequalityofdelta} (actually equality here) with probability $1$ based on Bellman equations.

Now suppose inequality \ref{inequalityofdelta} is true for any $k\in[K]$, $x\in\mathcal{S}$, for any $h'=\tau_h^k(x, a)$ that has $a\in A_h^k$:
\[\max_{y\in A_{\tau_h^k(x,a)}^k}\left|(Q_{\tau_h^k(x,a)}^k-Q_{\tau_h^k(x,a)}^*)( y)\right|\le \frac{\delta_{\tau_h^k(x,a)}}{\sqrt{k-1}},\forall a\in A_h^k, \text{ w.h.p.}\]
Now we induct on the previous stage $h'=h$. By Lemma \ref{martingale}, with probability at least $1-1/(AT)^8$
\begin{equation*}
\begin{aligned}
&\max_{y\in A_h^k}\Big|(Q_{h}^{k}-Q_{h}^{*})( y)\Big|\leq \max_{a\in A_h^k}\bigg\{\alpha_{k-1}^{0} H\\
&+\sum_{i=1}^{k-1} \alpha_{k-1}^{i}\bigg[\left(V_{\tau_h^i(x, a)}^{i+1}-V_{\tau_h^i(x, a)}^{*}\right)\left({x^i_{\tau_h^i(x, a)}}'\right)+\tilde{r}^i_{h,\tau_h^i(x, a)}-\tilde{r}^*_{h, \tau_h^i(x, a)}\bigg]+c \sqrt{\frac{H^{3} \iota}{k-1}}\bigg\}
\end{aligned}
\end{equation*}

Since based on our inductive hypothesis,we have 
\begin{equation*}
\begin{split}
    \max_{a\in A_h^k}\Big[\left(V_{\tau_h^i(x, a)}^{i+1}-V_{\tau_h^i(x,a)}^{*}\right)\left({x^i_{\tau_h^i(x, a)}}'\right)+\tilde{r}^i_{h,\tau_h^i(x,a)}-\tilde{r}^*_{h,\tau_h^i(x,a)}\Big]\\
    \le\max_{y\in A_{\tau_h^i(x,a)}^i}|(Q_{\tau_h^i(x,a)}^{i+1}-Q_{\tau_h^i(x,a)}^{*})(y)|\le \frac{\delta_{\tau_h^i(x,a)}}{\sqrt{i}}
\end{split}
\end{equation*} 
then
\begin{equation}
\max_{y\in A_h^k}\Big|(Q_{h}^{k}-Q_{h}^{*})( y)\Big|\le \max_{a\in A_h^k}\Bigg\{\alpha_{k-1}^{0}H+(\sum_{i=1}^{k-1} \alpha_{k-1}^{i}\cdot \frac{\delta_{\tau_h^i(x,a)}}{\sqrt{i}})+c \sqrt{\frac{H^{3} \iota}{k-1}}\Bigg\}.
\end{equation}

We can bound $\alpha_{k-1}^{0}$ by $\frac{1}{\sqrt{k}}$, and bound $\sum_{i=1}^{k-1} \alpha_{k-1}^{i}\cdot \frac{\delta_{\tau_h^i(x,a)}}{\sqrt{i}}$ by $\frac{1+1/H}{\sqrt{k-1}}\delta_{\tau_h^i(x,a)}$ using Lemma \ref{weights}:
\begin{equation}
\begin{aligned}
\max_{y\in A_h^k}\Big|(Q_{h}^{k}-Q_{h}^{*})(y)\Big|&\le \frac{1}{\sqrt{k}}H+\frac{1+1/H}{\sqrt{k-1}}\delta_{\tau_h^i(x,a)}+c \sqrt{\frac{H^{3} \iota}{k}}\\
&\le \frac{1}{\sqrt{k-1}}H+\frac{1+1/H}{\sqrt{k-1}}\delta_{h+1}+c \sqrt{\frac{H^{3} \iota}{k-1}} \hspace{1em}=\frac{\delta_h}{\sqrt{k-1}}
\end{aligned}
\end{equation}
where the second inequality is because $\tau_h^i(x,a)\ge h+1$ and $\delta_h$'s is a decreasing sequence. 
The last equality is true based on the recursive definition of $\delta_h$. 
\end{proof}

\begin{proof}{\bf (Lemma \ref{insideOPT})}
Recall for any $(x, h, k)$, $y_h^{k*}=\argmax_{y\in A_h^{k}}Q_h^k(y)$ in \emph{HQL}. 
Suppose $y_h^* \not\in A_h^k$, then $Q_h^k(y_h^*)<Q_h^k( y_h^{k*})-\frac{8\sqrt{H^5\iota}}{\sqrt{k-1}}=Q_h^k(x,y_h^{k*})-\frac{2\delta_h}{\sqrt{k-1}}$. Then we need either $Q_h^k( y_h^*)<Q_h^*( y_h^*)-\frac{\delta_h}{\sqrt{k-1}}$ or $Q_h^k(y_h^{k*})>Q_h^*( y_h^{k*})+\frac{\delta_h}{\sqrt{k-1}}$. Thus by Lemma \ref{delta}, Prob($y_h^* \not\in A_h^k(x)$)$\le\frac{1}{(AT)^5}$.
\end{proof}

\begin{proof}{\bf (Lemma \ref{within})}
Lemma \ref{delta} says for any $y\in A_h^k$, our estimated $Q_h^k(y)$ differs from the optimal value $Q_h^*(y)$ by at most $\frac{\delta_h}{\sqrt{k-1}}$ with high probability at least $1-\frac{1}{(AT)^5}$. Therefore, the optimal Q-value of the optimal policy's action $Q^*( y_h^*)$ is at most $\frac{\delta_h}{\sqrt{k-1}}$ more than the estimated Q-value of our estimated best arm $Q_h^k( y_h^{k*})$, with high probability at least $1-\frac{1}{(AT)^5}$. Any action we take in $A_h^k$ has an estimated Q-value no more than $\frac{8\sqrt{H^5\iota}}{\sqrt{k-1}}=\frac{2\delta_h}{\sqrt{k-1}}$ lower than $Q_h^k( y_h^{k*})$ base on our algorithm. Therefore, the optimal Q-value of the optimal policy's action $Q^*( y_h^*)$ is at most $\frac{3\delta_h}{\sqrt{k-1}}$ more than the estimated Q-value of any action $y\in A_h^k(x)$, with high probability at least $1-\frac{1}{(AT)^5}$. Then again, by Lemma \ref{delta}, we know that the optimal Q-value of the optimal policy's action $Q^*( y_h^*)$ is at most $\frac{4\delta_h}{\sqrt{k-1}}$ more than the optimal Q-value of any action in $ A_h^k$, with high probability at least $1-\frac{2}{(AT)^5}$. 
\end{proof}

\begin{proof}{\bf (Lemma \ref{outside})}
From Lemma \ref{insideOPT}, we know that with with probability at least $1-\frac{1}{(AT)^5}$, the optimal action is in the running set, which is inaccessible. Then recall the assumptions that the value functions are concave and that the feasible action set at any time is an interval of the form $\mathcal A\cap [a,\infty)$ for some $a$ dependent on the state. So if we cannot play in the running set, then the running set, and hence w.h.p. the true optimal action, is contained in $(-\infty,a)$. By concavity, this implies that the closest feasible action to the running set is optimal in this case w.p. at least $1-\frac{1}{(AT)^5}$. 
\end{proof}


\section{Missing Proofs for Inventory Control}
\label{invappendix}
We gave a more detailed description of the backlogged model and the lost-sales model of the episodic stochastic inventory control problems.

 \begin{lemma}
For any $h\in [H]$, the optimal $V$-value function $V^*_h(x)$ is concave in $x$, and the optimal $Q$-value function $Q^*_h(y)$ is concave in $y$. This is true for the lost sales and the backlogged models.
\label{convexity}
\end{lemma}
\begin{proof} {\bf (Lemma \ref{convexity})}
We prove this by backward induction. The base case is $Q^*_H(x,y)$ and $V^*_H(x)$. Since $Q^*_H(y)$ is just the expectation of a one time reward for the last period, we know that it is $Q^*_H(x, y)=r_H(x,y, D_H)=-[o_H(y-D_H)^++p_H\min(y,D_H)]$. This function is obviously concave in $y$. Note that the Q-values are not affected by $x$ for the inventory control problems. Since $V^*_H(x)=\max_{y\ge x}Q^*_H(x,y)$, the graph of $V^*_H(x)$ is constant on the left side of $x=\argmax_{y}Q^*_H(x,y)$, and then goes down with a slope of $o_H$ on the right side of $x=\argmax_{y}Q^*_H(x,y)$. So $V^*_H(x)$ is obviously also concave. 

Now suppose $Q^*_{h+1}(x,y)$ and $V^*_{h+1}(x)$ are concave. It remains to show concavity of $Q^*_{h}(x,y)$ and $V^*_{h}(x)$.

We know $Q^*_{h}(x,y)=\mathbb{E}[V^*_{h+1}(y-D_{h})+r_{h}(x,y,D_h)]$. We know $r_{h}(x,y,D_h)$ is concave in $y$ for the same reason that $Q^*_H(x,y)$ is concave. We know that $V^*_{h+1}(x)$ is concave in $x$ from our induction hypothesis, which means $V^*_{h+1}(y-D_{h})$ is concave in $y$ for any value of $D_h$. Therefore, $\mathbb{E}[V^*_{h+1}(y-D_{h})+r_h]$ is also concave, being a weighted average of concave functions. So we know $Q^*_h(x,y)$ is also concave in $y$. 
Then again $V^*_h(x)=\max_{y\ge x}Q^*_h(x,y)$ is concave for the same reason why $V^*_H(x)$ is concave.
\end{proof}

\begin{proof}{\bf (Assumption of $0$ Purchasing Costs)}
We want to show that for the episodic lost-sales (and similarly for the backlogged) model, we can amortize the unit purchasing costs $c_h$ into unit holding costs $o_h$ and unit lost-sales penalty $p_h$, so that without loss of generality we can assume $0$ purchasing costs.
\begin{equation}
\forall h\ge 2, y_h-x_h=y_h-D_h+D_h-x_h=(y_h-D_h)^+-(D_h-y_h)^++D_h-(y_{t-1}-D_{t-1})^+
\label{y-x}
\end{equation}

Let $c_h$ denote the unit purchasing cost, then the total sum of costs starting from stage 2 is
\begin{equation*}
\begin{aligned}
&\sum_{h=2}^H \Big( c_h(y_h-x_h)+o_h(y_h-D_h)^++p_h(D_h-y_h)^+\Big)\\
=&\sum_{h=2}^H \Big( c_hD_h-c_h(y_{h-1}-D_{h-1})^++(o_h+c_h)(y_h-D_h)^++(p_h-c_h)(D_h-y_h)^+\Big)\\
\end{aligned}
\end{equation*}
And the cost of stage 1 is equal to $o_1(y_1-D_1)^++p_1(D_1-y_1)^++c_1\big((y_1-D_1)^+-(D_1-y_1)^++D_1-x_1\big)$.

Let $c_{H+1}\ge 0$ denote the salvage price at which we sell the remaining inventory $(y_H-D_H)^+$ at the end of each episode. Then the total sum of costs from stage 1 to H is
\begin{equation*}
\begin{aligned}
&\sum_{h=2}^H \Big( c_hD_h-c_h(y_{h-1}-D_{h-1})^++(o_h+c_h)(y_h-D_h)^++(p_h-c_h)(D_h-y_h)^+\Big)\\
&+c_1(y_1-D_1)^+-c_1(D_1-y_1)^++c_1D_1-c_1x_1+o_1(y_1-D_1)^++p_1(D_1-y_1)^+-c_{H+1}(y_H-D_H)^+\\
&=\sum_{h=2}^H \Big( c_hD_h-c_h(y_{h-1}-D_{h-1})^++(o_h+c_h)(y_h-D_h)^++(p_h-c_h)(D_h-y_h)^+\Big)\\
&+c_1(y_1-D_1)^+-c_1(D_1-y_1)^++c_1D_1-c_1x_1+o_1(y_1-D_1)^++p_1(D_1-y_1)^+-c_{H+1}(y_H-D_H)^+\\
&=\sum_{h=1}^H c_hD_h+ \sum_{h=1}^{H} \Big((o_h+c_h-c_{h+1})(y_h-D_h)^++(p_h-c_h)(D_h-y_h)^+\Big)-c_1x_1\\
\end{aligned}
\end{equation*}

Since $\sum_{h=1}^H c_hD_h$ and $-c_1x_1$ are fixed costs independent of our action, we can take them out of our consideration. Then we can effectively consider the cost of each stage $h$ is just $o'_h(y_h-D_h)^++p_h'(D_h-y_h)^+$, where $o'_h=o_h+c_h-c_{h+1}$ is the adjusted holding cost, and $p_h'=p_h-c_h$ is the adjusted lost-sales penalty. 
\end{proof}

\section{Comparison with Existing Q-Learning Algorithms}
\label{priorQ}
For \cite{jin2018q}, suppose we discretize the state and action space optimally with step-size $\epsilon_1$ to apply \cite{jin2018q} to the backlogged/lost-sales episodic inventory control problem with continuous action and state space. Then the $\operatorname{Regret}_{gap}$ we get is $\epsilon_1 T$. Applying the results of \cite{jin2018q}, their $\operatorname{Regret}_{MDP}$ is $\mathcal{O}(\sqrt{H^3SAT\iota})=\mathcal{O}(\sqrt{\frac{1}{\epsilon_1}\cdot \frac{1}{\epsilon_1} T\iota})$. To minimize $\operatorname{Regret}_{total}$, we balance the $\operatorname{Regret}_{MDP}$ and $\operatorname{Regret}_{gap}$ by setting $\sqrt{\frac{1}{\epsilon_1}\cdot \frac{1}{\epsilon_1} T}=\epsilon_1 T$, which gives $\epsilon_1=\frac{1}{T^{1/4}}$, giving us an optimized regret bound of $\mathcal{O}(T^\frac{3}{4}\sqrt{H^3\log T})$.

For \cite{1912.06366}, suppose we discretize the state and action space optimally with step-size $\epsilon_2$ to apply \cite{1912.06366} to the backlogged/lost-sales episodic inventory control problem. We also optimize aggregation using the special property of these inventory control problems that the Q-values only depend on the action not the state, so we aggregate all the state-action pairs $(x_1,y), (x_2,y)$ into one aggregated state-action pair. This $0$-error aggregation helps reduce the aggregated state-action space. Then the optimized regret bound in \cite{1912.06366} is $\mathcal{O}(\sqrt{H^4\frac{1}{\epsilon}T\log T}+\epsilon T)$. We minimize $\operatorname{Regret}_{total}$ by balancing the two terms and take $\epsilon=\frac{1}{T^{1/3}}$, obtaining an optimized regret bound of $\mathcal{O}(T^\frac{2}{3}\sqrt{H^4\log T})$.

\section{Missing Proofs for \emph{FQL} }
\label{FQLappendix}

For the proof for \emph{FQL}, we adopt similar notations and flow of the proof in \cite{jin2018q} (but adapted to our full-feedback setting) to facilitate quick comprehension for readers who are familiar with \cite{jin2018q}.

Like \cite{jin2018q}, we use $\left[\mathbb{P}_{h} V_{h+1}\right](x, y):=\mathbb{E}_{x^{\prime} \sim \mathbb{P}(\cdot | x, y)} V_{h+1}\left(x^{\prime}\right)$. Then the Bellman optimality equation becomes $Q_{h}^{*}(x, y)=\left(r_{h}+\mathbb{P}_{h} V_{h+1}^{*}\right)(x,y)$.

Similar to Equation \ref{updateHQL} but without ``skipping'', \emph{FQL}  updates the $Q$ values in the following way for any $(x,y)\in\mathcal{A}$ at any time step:
\begin{equation}
    Q_h^{k+1}(x,y)\leftarrow (1-\alpha_{k})Q_h^{k}(x,y)+\alpha_{k}[r_h^{k+1}(x,y)+ V_{h+1}^{k}(x_{h+1})]
\end{equation}

Then by the definition of weights $\alpha_t^k$, we have
\begin{equation}
    Q_{h}^{k}(x,y)=\alpha_{k-1}^{0} H+\sum_{j=1}^{k-1} \alpha_{k-1}^{j}\left[r_{h}^j(x,y)+V_{h+1}^{j}\left(x_{h+1}^{j}\right)\right]
    \label{Qupdate}
\end{equation}

The following two lemmas are variations of Lemma \ref{Qdiff} and Lemma \ref{martingale}.

\begin{lemma}
For any $(x,y, h, k)\in \mathcal{S}\times\mathcal{A}\times [H]\times [K]$, we have
\begin{equation*}
\begin{aligned}
    \left(Q_{h}^{k}-Q_{h}^{*}\right)(x,y)&=\alpha_{k-1}^{0}\left(H-Q_{h}^{*}(x,y)\right)\\
    &+\sum_{i=1}^{k-1} \alpha_{k-1}^{i}\left[\left(V_{h+1}^{i}-V_{h+1}^{*}\right)\left(x_{h+1}^{i}\right)+{r}^i_h-\mathbb{E}[{r}^i_h]+\left[\left(\hat{\mathbb{P}}_{h}^{i}-\mathbb{P}_{h}\right) V_{h+1}^{*}\right](x,y)\right]
\end{aligned}
\end{equation*}
\label{bigFQLQdiff}
\end{lemma}

\begin{proof}
From the Bellman optimality equation $Q_{h}^{*}(x,y)=\mathbb{E}[r_{h}(x,y)]+\mathbb{P}_{h} V_{h+1}^{*}(x,y)$, our notation $\left[\hat{\mathbb{P}}_{h}^{i} V_{h+1}\right](x, y):=V_{h+1}\left(x_{h+1}^{i}\right)$, and the fact that $ \sum_{i=0}^{k-1} \alpha_{k-1}^{i}=1$, we have
\[Q_{h}^{*}(x,y)=\alpha_{k-1}^{0} Q_{h}^{*}(x,y)+\sum_{i=1}^{k-1} \alpha_{k-1}^{i}\left[\mathbb{E}[r_{h}^{i}(x, y)]+\left(\mathbb{P}_{h}-\hat{\mathbb{P}}_{h}^{i}\right) V_{h+1}^{*}(x, y)+V_{h+1}^{*}\left(x_{h+1}^{i}\right)\right]\]

Subtracting Equation \ref{Qupdate} from this equation gives us Lemma \ref{bigFQLQdiff}.
\end{proof}

\begin{lemma}
For any $p\in(0,1)$, with probability at least $1-p$, for any $(x, y, h, k)\in \mathcal{S}\times\mathcal{A}\times [H]\times [K]$, let $\iota=\log(SAT/p)$, we have for some absolute constant $c$:
\begin{equation}
    0 \leq\left(Q_{h}^{k}-Q_{h}^{*}\right)(x, y) \leq \alpha_{k-1}^{0} H+\sum_{i=1}^{k-1} \alpha_{k-1}^{i}\left(V_{h+1}^{i}-V_{h+1}^{*}\right)\left(x_{h+1}^{i}\right)+c \sqrt{\frac{H^{3} \iota}{k-1}}
\end{equation}
\label{bigFQLMartingale}
\end{lemma}
\begin{proof}
For any $i\in[k]$, recall that episode $i$ is the episode where the state-action pair $(x,y)$ was updated at stage $h$ for the $i$th time. Let $\mathcal{F}_h^i$ be the $\sigma$-field generated by all the random variables until episode $i$, stage $h$. Then for any $\tau\in[K]$, $\left([(\hat{\mathbb{P}}_{h}^{i}-\mathbb{P}_{h}) V_{h+1}^{*}](x, y)+{r}^i_h-\mathbb{E}[{r}^i_h]\right)_{i=1}^{\tau}$ is a martingale difference sequence w.r.t. the filtration $\left\{\mathcal{F}_h^{i}\right\}_{i \geq 0}$. Then by Azuma-Hoeffding Theorem, we have that with probability at least $1-p/SAT$:
\begin{equation}
\quad\left|\sum_{i=1}^{k-1} \alpha_{k}^{i}\cdot\left[\left(\hat{\mathbb{P}}_{h}^{i}-\mathbb{P}_{h}\right) V_{h+1}^{*}\right](x, y)+{r}^i_h-\mathbb{E}[{r}^i_h]\right| \leq \frac{c H}{2} \sqrt{\sum_{i=1}^{k-1}\left(\alpha_{k-1}^{i}\right)^{2} \cdot \iota} \leq c \sqrt{\frac{H^{3} \iota}{k-1}}
    \label{azumaFQL}
\end{equation}
for some constant $c$. 

Now we union bound over states, actions and times, we see that with probability at least $1-p$, we have 
\begin{equation}
    \left|\sum_{i=1}^{k-1} \alpha_{k=1}^{i}\left[\left(\hat{\mathbb{P}}_{h}^{k_{i}}-\mathbb{P}_{h}\right) V_{h+1}^{*}\right](x, y)+{r}^i_h-\mathbb{E}[{r}^i_h]\right| \leq c \sqrt{\frac{H^{3} \iota}{k-1}} 
    \label{unionboundFQL}
\end{equation}

Then the right-hand side of Lemma \ref{bigFQLMartingale} follows from Lemma \ref{bigFQLQdiff} and Inequality \ref{unionboundFQL}. The left-hand side also follows from Lemma \ref{bigFQLQdiff} and Inequality \ref{unionboundFQL} using induction on $h=H, H-1, \dots, 1$.
\end{proof}

\begin{proof}{\bf (Theorem \ref{HQLresult})}
Define $\Delta_{h}^{k}:=\left(V_{h}^{k}-V_{h}^{\pi_{k}}\right)\left(x_{h}^{k}\right)$ and $\phi_{h}^{k}:=\left(V_{h}^{k}-V_{h}^{*}\right)\left(x_{h}^{k}\right)$.

By Lemma \ref{azumaFQL}, with $1-p$ probability, $Q_{h}^{k} \geq Q_{h}^{*}$ and thus $V_{h}^{k} \geq V_{h}^{*}$. Thus the total regret can be upper bounded: 
\[\operatorname{Regret}(K)=\sum_{k=1}^{K}\Big(V_{1}^{*}-V_{1}^{\pi_{k}}\Big)(x_{1}^{k}) \leq \sum_{k=1}^{K}\left(V_{1}^{k}-V_{1}^{\pi_{k}}\right)(x_{1}^{k})=\sum_{k=1}^{K} \Delta_{1}^{k}\]

The main idea of the rest of the proof is to upper bound $\sum_{k=1}^{K} \Delta_{h}^{k}$ by the next step $\sum_{k=1}^{K} \Delta_{h+1}^{k}$, which gives a recursive formula to obtain the total regret. Let $y_h^k$ denote the base stock levels taken at stage $h$ of episode $k$, which means $y_h^k=\argmax Q_h^k(y')$.
\begin{equation}
    \begin{aligned}
\Delta_{h}^{k} =&\left(V_{h}^{k}-V_{h}^{\pi_{k}}\right)(x_{h}^{k}) \stackrel{(1)}{\le }\left(Q_{h}^{k}-Q_{h}^{\pi_{k}}\right)( x_{h}^{k}, y_{h}^{k}) \\
=&\left(Q_{h}^{k}-Q_{h}^{*}\right)(x_{h}^{k}, y_{h}^{k})+\left(Q_{h}^{*}-Q_{h}^{\pi_{k}}\right)(x_{h}^{k}, y_{h}^{k}) \\
 \stackrel{(2)}{\le }&\alpha_{k-1}^{0} H+\sum_{i=1}^{k-1} \alpha_{k-1}^{i} \phi_{h+1}^{i}+c \sqrt{\frac{H^{3} \iota}{k-1}}+\left[\mathbb{P}_{h}\left(V_{h+1}^{*}-V_{h+1}^{\pi_{k}}\right)\right](x_{h}^{k}, y_{h}^{k}) \\
 =&\alpha_{k-1}^{0} H+\sum_{i=1}^{k-1} \alpha_{k-1}^{i} \phi_{h+1}^{i}+c \sqrt{\frac{H^{3} \iota}{k-1}}+\left[\left(\mathbb{P}_{h}-\hat{\mathbb{P}}_{h}^{k}\right)\left(V_{h+1}^{*}-V_{h+1}^{\pi_k}\right)\right](x_{h}^{k}, y_{h}^{k})\\
&+(V_{h+1}^*-V_{h+1}^{\pi_k})(x_{h+1}^k)\\
 \stackrel{(3)}{=} &\alpha_{k-1}^{0} H+\sum_{i=1}^{k-1} \alpha_{k-1}^{i} \phi_{h+1}^{i}+c \sqrt{\frac{H^{3} \iota}{k-1}}-\phi_{h+1}^{k}+\Delta_{h+1}^{k}+\xi_{h+1}^{k}
\end{aligned}
\label{delta_FQL}
\end{equation}
where $\xi_{h+1}^{k}:=\left[\left(\mathbb{P}_{h}-\hat{\mathbb{P}}_{h}^{k}\right)\left(V_{h+1}^{*}-V_{h+1}^{\pi_k}\right)\right]\left(x_{h}^{k}, y_{h}^{k}\right)$ is a martingale difference sequence. Inequality (1) holds because $V_{h}^{k}\left(x_{h}^{k}\right) \leq \max_{\text{feasible } y' \text{ given }x} Q_{h}^{k}\left(x_{h}^{k}, y^{\prime}\right)=Q_{h}^{k}\left(x_{h}^{k}, y_{h}^{k}\right)$, and Inequality (2) holds by Lemma \ref{bigFQLMartingale} and the Bellman equations. Inequality (3) holds by definition $\Delta_{h+1}^{k}-\phi_{h+1}^{k}=\left(V_{h+1}^{*}-V_{h+1}^{\pi_{k}}\right)\left(x_{h+1}^{k}\right)$.

In order to compute $\sum_{k=1}^{K} \Delta_{1}^{k}$, we need to first bound the first term in Equation \ref{delta_FQL}. Since $\alpha_{k}^0=0, \forall k\ge 1$, we know that $\sum_{k=1}^{K} \alpha_{k-1}^{0} H \le H$.

Now we bound the sum of the second term in Equation \ref{delta} over the episodes by regrouping:
\begin{equation}
    \sum_{k=2}^{K} \sum_{i=1}^{k-1} \alpha_{k-1}^{i} \phi_{h+1}^{i} \leq \sum_{i=1}^{K-1} \phi_{h+1}^{i} \sum_{k=i+1}^{\infty} \alpha_{k-1}^{i}
    \leq \sum_{i=1}^{K-1} \phi_{h+1}^{i} \sum_{k'=i}^{\infty} \alpha_{k'}^{i}
    \leq\left(1+\frac{1}{H}\right) \sum_{k=1}^{K} \phi_{h+1}^{k}
    \label{regroup}
\end{equation}
where the last inequality uses $\sum_{t=i}^{\infty} \alpha_{t}^{i}=1+\frac{1}{H} \text { for every } i \geq 1$ from Lemma \ref{weights}.

Plugging the above Equation \ref{regroup} and $\sum_{k=1}^{K} \alpha_{k}^{0} H \le H$ back into Equation \ref{delta}, we have:
\begin{equation}
\begin{aligned}
\sum_{k=1}^{K} \Delta_{h}^{k} &\le H+\sum_{k=2}^{K} \Delta_{h}^{k} \\
& \leq H+ H+\left(1+\frac{1}{H}\right) \sum_{k=1}^{K} \phi_{h+1}^{k}-\sum_{k=2}^{K} \phi_{h+1}^{k}+\sum_{k=2}^{K} \Delta_{h+1}^{k}+\sum_{k=2}^{K}c \sqrt{\frac{H^{3} \iota}{k-1}}+\sum_{k=2}^{K}\xi_{h+1}^{k} \\
& \le 2H+\phi_{h+1}^{1}+\frac{1}{H}\sum_{k=2}^{K} \phi_{h+1}^{k}+\sum_{k=2}^{K} \Delta_{h+1}^{k}+\sum_{k=2}^{K}c \sqrt{\frac{H^{3} \iota}{k-1}}+\sum_{k=2}^{K}\xi_{h+1}^{k} \\
& \leq 3H+\left(1+\frac{1}{H}\right) \sum_{k=2}^{K} \Delta_{h+1}^{k}+\sum_{k=2}^{K}c \sqrt{\frac{H^{3} \iota}{k-1}}+\sum_{k=}^{K}\xi_{h+1}^{k} 
\end{aligned}
\end{equation}
where the last inequality uses $\phi_{h+1}^k\le\Delta_{h+1}^k$. By recursing on $h=1,2, \dots, H$, and because $\Delta_{H+1}^K=0$:
\begin{equation*}\sum_{k=1}^{K} \Delta_{1}^{k} \leq \mathcal{O}\left(\sum_{h=1}^{H} \sum_{k=1}^{K}\big(c \sqrt{\frac{H^{3} \iota}{k-1}}+\xi_{h+1}^{k}\big)\right)\end{equation*}
where $\sum_{h=1}^H\sum_{k=1}^{K}c \sqrt{\frac{H^{3} \iota}{k-1}}= \mathcal{O}(H\sqrt{H^3\log(SAT/p)}\sqrt{ K})=\tilde{\mathcal{O}}(\sqrt{H^4T})$.

On the other hand, by Azuma-Hoeffding inequality, with probability $1-p$, we have 
\begin{equation}
\left|\sum_{h=1}^{H} \sum_{k=1}^{K} \xi_{h+1}^{k}\right|=\left|\sum_{h=1}^{H} \sum_{k=1}^{K}\left[\left(\mathbb{P}_{h}-\hat{\mathbb{P}}_{h}^{k}\right)\left(V_{h+1}^{*}-V_{h+1}^{\pi_k}\right)\right]\left(x_{h}^{k}, y_{h}^{k}\right)\right| \leq c H \sqrt{T_{l}} \hspace{0.2cm}\le \tilde{\mathcal{O}}(\sqrt{H^4T})
\end{equation}
which establishes $\sum_{k=1}^{K} \Delta_{1}^{k} \leq \tilde{\mathcal{O}}(H^2\sqrt{T})$. 
\end{proof}

\section{More Numerical Experiments}
\label{numericappend}

We show more numerical experiment results to demonstrate the performance of \emph{FQL} and \emph{HQL}. In Table \ref{alltable2}, we use again the backlogged model to compare \emph{FQL} and \emph{HQL} against \emph{OPT}, \emph{Aggregated QL} and \emph{QL-UCB}, but with a different set of parameter than in Section \ref{numerical}. In Table \ref{lostsalestable} and \ref{lostsalestable2}, we use the lost-sales model to compare \emph{HQL} against \emph{OPT}, \emph{Aggregated QL} and \emph{QL-UCB}.

For Tables \ref{alltable2} and \ref{lostsalestable}, we make the demand distribution less adversarial: with each step in the episode, we have demands that are increasing in expectation. However, we let the upper bound of base-stock levels increase with the episode length $H$, which is more adversarial. For Table \ref{lostsalestable2}, we use the same demand distributions and base-stock upper bound as in Table \ref{alltable} in Section \ref{numerical}. 

We run each experimental point $300$ times for statistical significance.

\noindent\begin{minipage}{.5\linewidth}
 {\bf Episode length}: $H=1, 3, 5$.
 
 {\bf Number of episodes}: $K=100,500,2000$.
 
{\bf Demands}: $D_h\sim U[0,1]+h$.

\end{minipage}%
\begin{minipage}{.5\linewidth}
 {\bf Holding cost}: $o_h=2$.
 
 {\bf Backlogging cost}: $b_h=10$.
 
{\bf Action space}: $[0, \frac{1}{20}, \frac{2}{20}, \dots, 2H]$.
\end{minipage}


\vspace{-0.4cm}

\begin{center}
\begin{table}[h]
 \caption{Comparison of cumulative costs for backlogged episodic inventory control with less adversarial demands and increasing base-stock upper bounds }

 \begin{tabular}{p{0.2cm}p{0.8cm}p{0.8cm}p{0.6cm}p{0.8cm}p{0.6cm}p{0.8cm}p{0.7cm}p{1cm}p{0.9cm}p{1cm}p{0.6cm}} 
\multicolumn{2}{c}{} & \multicolumn{2}{c}{ \emph{OPT} } &  \multicolumn{2}{c}{\emph{FQL} } & \multicolumn{2}{c}{\emph{HQL} } & \multicolumn{2}{c}{ \emph{Aggregated QL}} & \multicolumn{2}{c}{ \emph{QL-UCB}}\\  
{\bf H} & {\bf K }  & {\bf mean} & {\bf SD} &  {\bf mean} & {\bf SD} &  {\bf mean} & {\bf SD}  &  {\bf mean} & {\bf SD} &  {\bf mean} & {\bf SD} \\   
 \hline
 \multirow{3}{*}{1} & $100 $ & 89.1 & 3.8& 97.1 & 5.5  & 117.3 & 16.8 & 160.1 & 8.3 & 327.5 &18.8\\ 
  & $500 $&  420.2   & 4.2 & 431.2 & 4.2  & 507.8 & 45.6 & 732.7 & 22.1 & 825.4 & 10.9 \\ 
 
 & $2000 $&  1669.8  & 4.8 &  1691.2 & 6.6  & 1883.6 & 99.7  & 2546.2& 32.6& 2952.1 & 19.9\\ 
 \hline

 \multirow{3}{*}{3} & $100$ & 253.0 & 6.6  & 304.6 & 9.6  & 423.8 & 15.4 & 510.9& 14.4& 1712.0 & 19.1\\
 
 & $500$& 1252.4 & 7.0 & 1314.3 & 11.9  & 1611.0 & 43.9  & 1703.2 & 16.1 & 4603.7 & 101.6 \\
 & $2000$ & 5056.2 & 6.5  & 5128.7 & 10.2 & 5702.8 & 104.7 & 6188.0  & 14.1  & 15088.6  & 132.0\\
 \hline
 \multirow{3}{*}{5} & $ 100$ & 415.9 & 6.4 & 543.6 & 11.0  & 762.4 & 30.0 & 3011.8 & 1294.6 & 6101.9 & 357.6 \\ 
& $500$  &  2077.1 & 12.7  & 2224.6 & 15.6  &  2746.3  &113.7 & 10277.1 & 6888.5 & 11763.6 & 2982.5  \\ 
& $ 2000$ &  8394.3 & 6.2 & 8557.2 & 11.1 &  9630.4 &  356.6 & 30489.8 & 31232.4 & 39873.8 & 7210.1\\ 
\bottomrule
\end{tabular}
 \label{alltable2}
\end{table}
\end{center}
\vspace{-0.6cm}

Again for \emph{Aggregated QL} from \cite{1912.06366} and for \emph{QL-UCB} from \cite{jin2018q}, we optimize by taking the Q-values to be only dependent on the action, thus reducing the state-action pair space. As in Section \ref{numerical}, we do not fine-tune the confidence interval for \emph{HQL} for different settings, but use a general formula $\sqrt{\frac{H\log(HKA)}{k}}$ as the confidence interval for all settings. We also do not fine-tune the \emph{UCB bonus} defined in \emph{QL-UCB} (see \cite{jin2018q}).

A caveat of \emph{Aggregated QL} from \cite{1912.06366} is that we need to know a good aggregation of the state-action pairs beforehand, which is usually unavailable for online problems. For using \emph{Aggregated QL} in Table \ref{alltable2} and \ref{lostsalestable}, we further aggregate the state and actions to be multiples of $1/2$. For using \emph{Aggregated QL} in Table \ref{lostsalestable2} (and also in Section \ref{numerical}), we further aggregate the state and actions to be multiples of $1$.

\noindent\begin{minipage}{.5\linewidth}
 {\bf Episode length}: $H=1, 3, 5$.
 
 {\bf Number of episodes}: $K=100,500,2000$.
 
{\bf Demands}: $D_h\sim U[0,1]+h$.

\end{minipage}%
\begin{minipage}{.5\linewidth}
 {\bf Holding cost}: $o_h=2$.
 
{\bf Lost-Sales Penalty}: $b_h=10$.
 
{\bf Action space}: $[0, \frac{1}{20}, \frac{2}{20}, \dots, 2H]$.
\end{minipage}

\vspace{-0.4cm}

\begin{center}
\begin{table}[h]
 \caption{Comparison of cumulative costs for lost-sales episodic inventory control with less adversarial demands and increasing base-stock upper bounds}
 \begin{tabular}{p{0.2cm}p{0.8cm}p{1cm}p{0.8cm}p{1cm}p{1.2cm}p{1cm}p{1cm}p{1cm}p{1cm}} 
\multicolumn{2}{c}{} & \multicolumn{2}{c}{ \emph{OPT} } & \multicolumn{2}{c}{\emph{HQL} } & \multicolumn{2}{c}{ \emph{Aggregated QL}} & \multicolumn{2}{c}{ \emph{QL-UCB}}\\  
{\bf H} & {\bf K }  &  {\bf mean} & {\bf SD} &  {\bf mean} & {\bf SD}  &  {\bf mean} & {\bf SD} &  {\bf mean} & {\bf SD} \\   
 \hline

 \multirow{3}{*}{1} & $100 $ &89.1 & 3.8 & 117.3 & 16.8   & 201.7 & 6.6 & 291.7 & 6.6\\ 
  & $500 $&   420.2 & 4.2  & 507.8 & 44.6 & 1002.8 & 4.0 & 1452.8 & 4.0 \\ 
 
 & $2000 $&  1669.8  & 4.8  & 1883.6 & 99.7  & 4012.1& 5.3 & 5812.1 & 5.3\\ 
 \hline

 \multirow{3}{*}{3} & $100$ & 253.0 & 6.6  & 443.8 & 65.9  & 1902.8 & 81.4 & 2071.4 & 29.9\\
 
 & $500$& 1252.4 & 7.0 & 1730.7 & 361.3   & 9534.0 & 379.7 & 10375.7 & 13.1 \\
 & $2000$ & 5056.2 & 6.5  & 6163.2 & 374.3 & 38139.6  & 1519.4  & 41504.9  & 22.6\\
 \hline
 \multirow{3}{*}{5} & $ 100$ & 415.9 & 6.4 & 780.6& 64.3 & 5716.6 & 153.0 & 5902.8 & 44.8 \\ 
& $500$  &  2077.1 & 12.7  & 2926.0 & 332.6   & 28510.7 & 764.3 & 29385.1 & 183.1  \\ 
& $ 2000$ &  8394.3 & 6.2 & 10560.1 & 1201.6 & 114010.7 & 3080.2 & 117481.6 & 727.6\\ 
\bottomrule
\end{tabular}
 \label{lostsalestable}
\end{table}
\end{center}

\vspace{2cm}
\newpage
\noindent\begin{minipage}{.5\linewidth}
 {\bf Episode length}: $H=1, 3, 5$.
 
 {\bf Number of episodes}: $K=100,500,2000$.
 
{\bf Demands}: $D_h\sim (10-h)/2+U[0,1]$.
\end{minipage}%
\begin{minipage}{.5\linewidth}
 {\bf Holding cost}: $o_h=2$.
{\bf Lost-Sales Penalty}: $b_h=10$.
{\bf Action Space}: $[0, \frac{1}{20}, \frac{2}{20}, \dots, 10]$.
\end{minipage}

\vspace{-0.2cm}

\begin{center}
\begin{table}[h]
 \caption{Comparison of cumulative costs for lost-sales episodic inventory control with the original demands and base-stock upper bounds}
 \begin{tabular}{p{0.2cm}p{0.8cm}p{1cm}p{0.8cm}p{1cm}p{1.2cm}p{1cm}p{1cm}p{1cm}p{1cm}} 
\multicolumn{2}{c}{} & \multicolumn{2}{c}{ \emph{OPT} } & \multicolumn{2}{c}{\emph{HQL} } & \multicolumn{2}{c}{ \emph{Aggregated QL}} & \multicolumn{2}{c}{ \emph{QL-UCB}}\\  
{\bf H} & {\bf K } & {\bf mean} & {\bf SD} &  {\bf mean} & {\bf SD}  &  {\bf mean} & {\bf SD} &  {\bf mean} & {\bf SD} \\  
 \hline

 \multirow{3}{*}{1} & $100 $ & 88.2 & 4.1 & 125.9 & 19.2   &  705.4 &  9.7 & 895.4 & 9.7\\ 
  & $500 $&   437 & 4.4 & 528.9 & 44.1 & 3506.1 & 4.4 & 4456.1 & 4.4 \\ 
 
 & $2000 $&  1688.9  & 2.8  & 1929.2 & 89.1  & 14005.6 & 6.6 & 17805.6 & 6.6\\ 
 \hline

 \multirow{3}{*}{3} & $100$ & 257.4 & 3.2  & 448.4 & 52.1  & 2405.6 & 9.1 & 2975.6 & 9.1\\
 
 & $500$& 1274.6 & 6.1 & 1746.7 & 239.9   &  12009.3 & 6.4 & 14859.3 & 6.4 \\
 & $2000$ & 4965.6 & 8.3  & 6111.2 & 918.2 & 47926.4  & 14.8  & 59326.4  & 14.8\\
 \hline
 \multirow{3}{*}{5} & $ 100$ & 421.2 & 3.3 & 774.6 & 51.8 & 4497.4 & 11.6 & 5447.4 & 11.6 \\ 
& $500$  &  2079.0 & 8.2  & 2973.9 & 299.9   & 22478.5 & 10.7 & 27228.5 & 10.7  \\ 
& $ 2000$ &  8285.7 & 8.3 & 10701.1 & 1207.5 & 89929.7 & 14.0 & 108929.7 & 14.0\\ 
\bottomrule
\end{tabular}
 \label{lostsalestable2}
\end{table}
\end{center}

As we can see in all of our experiments, \emph{FQL} and \emph{HQL} both perform very promisingly with significant advantage over the other two existing algorithms. \emph{FQL} stays consistently very close to the clairvoyant optimal in both the more adversarial and less adversarial settings for the backlogged model. \emph{HQL} catches up rather quickly to \emph{OPT} in all the settings for both the backlogged model and the lost-sales model.

\end{document}